\documentclass{ecai}
\usepackage{times}
\usepackage{graphicx}
\usepackage{latexsym}
\usepackage{amsmath,amssymb,amsfonts}
\usepackage{textcomp}
\usepackage{xcolor}
\usepackage{algorithm,algorithmic}
\usepackage{multirow}
\usepackage{booktabs}
\usepackage{array}
\usepackage{stfloats}
\usepackage{subfigure}
\usepackage{color}
\usepackage{caption2}
\usepackage[numbers,sort&compress]{natbib}
\usepackage{url}
\newenvironment{proof}{{\textbf{\textit{Proof:}}}}{}
\newcommand{\tabincell}[2]{\begin{tabular}{@{}#1@{}}#2\end{tabular}}

\begin{document}

\title{Homogeneous Online Transfer Learning with Online Distribution Discrepancy Minimization}

\author{
Yuntao Du \institute{Nanjing University,China, duyuntao@smail.nju.edu.cn} \and
Zhiwen Tan \institute{Nanjing University,China,yaoyueduzhen@outlook.com } \and
Qian Chen \institute{Nanjing University,China, mf1833007@smail.nju.edu.cn} \and \\
Yi Zhang \institute{Nanjing University,China, njuzhangyi@smail.nju.edu.cn} \and
Chongjun Wang \institute{Nanjing University,China, email: chjwang@nju.edu.cn}
}

\maketitle
\bibliographystyle{ecai}

\begin{abstract}
  Transfer learning has been demonstrated to be successful and essential in diverse applications, which transfers knowledge from related but different source domains to the target domain.
   Online transfer learning(OTL) is a more challenging problem where the target data  arrive in an online manner. Most OTL methods combine source classifier and target classifier directly by assigning a weight to each classifier, and adjust the  weights constantly.  However, these methods pay little attention to reducing the distribution discrepancy between domains. In this paper, we propose a novel online transfer learning method which seeks to find a new feature representation, so that the marginal distribution and conditional distribution discrepancy can be online reduced  simultaneously. We focus on online transfer learning  with multiple source domains and use the Hedge strategy to leverage knowledge from source domains.
   We analyze the theoretical properties of the proposed algorithm and provide an upper mistake bound.
   Comprehensive experiments on two real-world datasets show that our method outperforms state-of-the-art methods by a large margin.
  
  \end{abstract}
  
  \section{Introduction}
  Transfer learning\cite{b1} is able to transfer knowledge from labeled source domains to the target domain. In transfer learning problems, the source domain and the target domain have different distributions. The distribution discrepancy between source domain and target domain includes marginal distribution discrepancy and conditional distribution discrepancy\cite{b3,b4,b2}. In view of this challenge, many transfer learning methods are designed to reduce the distribution discrepancy between domains so that the  transfer performance can be improved \cite{b3,b4,b2}. Previous work explored various techniques for statistics matching, e.g, maximum mean discrepancy(MMD)\cite{b3,b4,b2}, Correlation Alignment (CORAL)\cite{b4_add} and have achieved remarkable effect.
  
  Most existing transfer learning methods focus on batch learning manners, where all the source and target data are given in advance. However, this assumption may not hold in real-world applications,  it is expensive to collect sufficient data at one time. Moreover, in some situations, the instances arrive in a sequential manner. Recently, Online transfer learning(OTL)\cite{b5} has attracted a lot of attention in the community of machine learning.  In this scenario, we can collect many labeled source data in advance, but we receive the target data in an online manner, i.e., we receive a target instance at each round. OTL aims at preforming an online task in the target domain by leveraging knowledge from some offline source domains. In online learning\cite{b14}, we need to train a classification model that can predict for the new arrival instance in the target domain, then we will receive its true label,  and we sequentially update the existing model based on the loss information between the true label and the predicted label.
  
  Existing online transfer learning methods focus on how to perform online learning in the target domain by leveraging knowledge from source domains \cite{b5,b6,b7}. They mostly adopt the ensemble learning based strategies,  i.e., combining the source classifier and the target classifier directly. \cite{b5} proposed to train a classifier with the help of the source domain, and adapted the classifier to new coming data by constantly adjusting the combination weights. On the basis of \cite{b5}, a lot of works focus on multi-source online transfer learning by using different technologies of multi-source combination, for example, binary graph\cite{b6}, selecting source domains adaptively\cite{b7}. Though achieving exciting performance, many previous works have proven that when the distribution discrepancy between domains is large, directly combining source classifier and target classifier can not perform well and even lead to negative transfer\cite{b1}. In other words, it is also crucial for online transfer learning to reduce distribution discrepancy, which is ignored by existing methods.

  \begin{table*}[ht]
   \caption{ Notations and descriptions in this paper}
    \centering
      \begin{tabular}{c|c|c|c}
  
      \toprule
      Notation & Description  & Notation & Description  \\
      \midrule[1pt]
      \midrule[1pt]
      $\mathcal{D}_{S_i}$                   & $i$-th source domain data                                   & $\mathcal{D}_T^l ,   \mathcal{D}_T^u $     &  labeled target domain data, unlabeled target domain data     \\
      $n_{S_{i}}$                           & \# $i$-th source domain data                                & $n_T^u$, $n_T^l$                           &   \# unlabeled target domain data / labeled target domain data       \\
      $n, K$                                & \# source domains / classes                                 & $\bar{c}_{S_i}$                            & the mean of the $i$-th source domain data               \\
      $A_i^t$                               & transformation matrix of the $i$-th domain at round \textit{t}                & $\bar{c}_{T}$                              & the mean of the target domain data                    \\
      $f_{S_i}$                             & the classifier of the $i$-th source domain                  & $\bar{c}_{S_i}^{k}$                        & the mean of the $k$-th class of the $i$-th source domain data \\
      $f_{T_i}$                             & the $i$-th classifier of the target domain                  & $\bar{c}_T^{k}$                            & the mean of the $k$-th class of the target domain  data        \\
      $w_{i,t}^k$                           & the $k$-th weight vector of the $i$-th classifier at round $t$            & $X_{s_i}^{p}$                              & projected data of the $i$-th source domain data            \\
      $F_t^k$                               &    the $k$-th component of the final classifier at round \textit{t}  & $x_{t}^{p_i}$                              & projected data of the $i$-th feature space for $x_t$      \\
      $u_i$                                 & the weight of the $i$-th source classifier                  & $T_w$                                      &time window            \\
      $v_i$                                 & the weight of the $i$-th target classifier                  & $\mu$                                      & regularization parameter          \\
      \bottomrule
      \end{tabular}%
  \label{t1}
  \end{table*}%

  In this paper, we investigate online transfer learning with multiple domains under homogeneous space, where there are some source domains, and the feature space of source domains and target domain is the same. We use ensemble learning based method i.e., we train a classifier in each domain then assign a weight to each classifier and take advantage of the Hedge strategy\cite{b30} to dynamically adjust the weight of each domain. For each source domain and the target domain, to reduce domain discrepancy in an online manner, we propose a novel online method, which jointly minimizes marginal distribution discrepancy and conditional distribution discrepancy  based on MMD in a linear feature transformation procedure. A linear transformation matrix is used to project the original source and target data onto a new space for each source domain and the target domain. This learning process is named as \textbf{\emph {online stage}}.

  In the above method, we can just use a random matrix as an initialization matrix, but it will drop the transfer performance because in the early peroid, the learning process is performed in an unsuited space. We design an \textbf{\emph{offline stage}} to meet this challenge. Unlabeled data are usually easy to collect in the target domain, with these unlabeled data and source data, we can use some classical offline MMD-based method, e.g, TCA\cite{b3}, JDA\cite{b4}, BDA\cite{b2} to get a good initialization matrix for each source domain and the target domain. After getting the transformation matrixes, we firstly project the original data to get the new feature representation. Then, in each source domain, a classifier is trained for transferring source knowledge.

  In general, we propose an algorithm called \textit{Homogeneous Online Transfer Learning with Online Distribution Discrepancy Minimization} (HomOTL-ODDM), which is consisted of the \textbf{\emph{offline stage}} and the \textbf{\emph{online stage}}. At the offline stage, an initialization transformation matrix and a classifier is trained in each source domain. At the online stage, for each source and target domain, an online classifier is trained in an online manner in new feature space. The transformation matrixes and the weights of classifiers are also updated. We theoretically analyze the proposed algorithm  and provide an upper mistake bound. Extensive experiments demonstrate that HomOTL-ODDM outperforms state-of-the-art methods by a large margin.

  To sum up, our contributions are mainly three-fold:

  1) We propose a novel method for multiple source online transfer learning by simultaneously reducing the marginal distribution  and the conditional distribution discrepancy between domains in a linear feature transformation procedure.
  
  2) We analyze the theoretical properties of the proposed algorithm and provide an upper mistake bound.
  
  3) We conduct extensive experiments on two real-world data sets  and the results validate the effectiveness of the proposed method.

  \section{Related work}
  
  \textbf{Transfer Learning:} The proposed method is mainly related to the feature-based transfer learning methods. \cite{b3,b17,b19} try to correct for the marginal distribution discrepancy between domains. \cite{b3} learns some transfer components across domains in a Reproducing Kernel Hilbert Space (RKHS) using \textit{Maximum Mean Discrepancy (MMD)}. A new parametric kernel is used to extract feature, which can dramatically minimize the distance between domain marginal distribution by projecting data onto the learned transfer components.  \cite{b23} tries to correct for the conditional distribution discrepancy between domains. Based on measures of divergence between the two domains, it penalizes features with large divergence, while improving the effectiveness of other less deviant correlated features. \cite{b4} first tries to correct for both the marginal and conditional distribution discrepancy simultaneously. Instead of treating these two discrepancy equally, \cite{b2} adaptively leverages the importance of the marginal and conditional distribution discrepancy. However, these methods can not be used in online manner directly because they are designed for offline manner. Our method is capable of reducing marginal and conditional distribution discrepancy simultaneously in an online manner.
  
  \textbf{Online Learning:}One of the most famous approaches is the Perceptron algorithm \cite{b25}, which updates the model by adding a new instance with some constant weight into the current set of support vectors when the instance is misclassified. Recently, many algorithms have been proposed based on maximum margin \cite{b9,b26,b27}. One example is the Passive-Aggressive (PA) algorithm\cite{b9}, which updates the model when a new instance is misclassified or its classification confidence is smaller than some predefined threshold. Specially, it can not only be used for binary classification problems but also for multi-class problems. More extensive surveys for online learning can be found in \cite{b14}.
  
  \textbf{Online Transfer Learning:}
  Unlike conventional transfer learning, online transfer learning assumes that labeled source   data are available in advance but the target data arrive sequentially. \cite{b5} proposes an online transfer learning algorithm for single source domain under homogeneous and heterogeneous space, which forms the final classifier by simply ensembling the source and target classifiers with dynamical weights. On the basis of \cite{b5},  \cite{b6} proposes an online transfer learning algorithm based on multiple homogeneous source domain, it assumes that the online target data is unlabeled.  On the premise that the target domain is labeled data, two online transfer learning algorithms based on multiple homogeneous or heterogeneous source domains are proposed in \cite{b7}, which also constitute the final classifier in an ensemble strategy. \cite{b33} researches on the application of online transfer learning in forecasting advertising volume. Existing methods have focused on how to transfer knowledge from source domains, however, they pay little attention to how to reduce the distribution discrepancy between domains. Instead, our method can reduce both together so that the transfer performance is improved.

  \begin{figure*}
  \begin{center}
  \includegraphics[width = 0.65\linewidth]{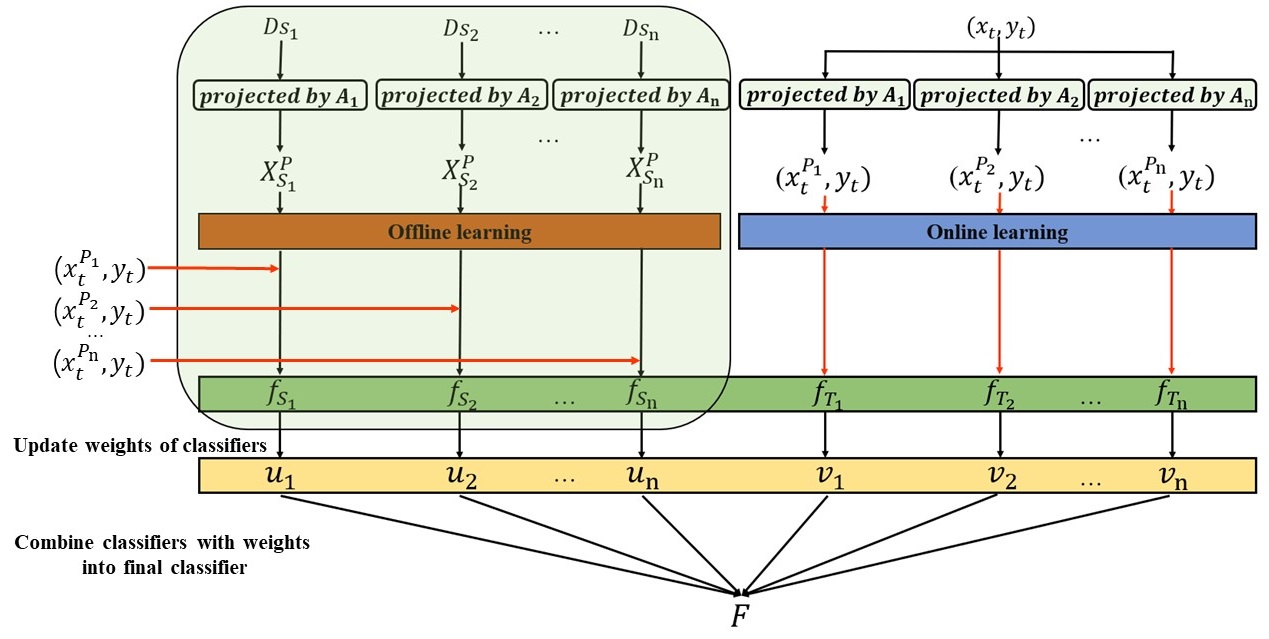}
  \end{center}
  \caption{ Structure of HomOTL-ODDM algorithm. It is composed of an the offline stage and an online stage. At the offline stage, $n$ source classifiers and initial transformation matrixes are trained in each source domain for knowledge transfer. At  the online stage, the target data is projected onto $n$ different feature space, and $n$ classifiers are trained in these new space. Each classifier is assigned a weight and the Hedge strategy is used to update the weights dynamically. Besides, we design an online method to update the transform matrixes, which can reduce the marginal and conditional distribution discrepancy between domains simultaneously.}
  \label{overall}
  \end{figure*}
  
  \section{problem definition}

  We begin with the definitions of terminologies, then frequently used notations are summarized in Table \ref{t1}.

  In multi-source OTL, $n$ source domains $D_S=\{D_{S_1},D_{S_2},...,D_{S_n}\}$ and a target domain $T$ are given. Each source domain $D_{S_i}=\{(x_j,y_j),j=1,2,...,n_{S_i}\}$ contains $n_{S_i}$ instances drawn from distribution $P_{S_i}(x)$, where $x_j \in \mathcal X_{S_i}$ with associated labels $y_j \in \mathcal Y_{S_i}$.
  Similarly, the target domain consists of $n_T^u$ offline unlabeled instances drawn from distribution $P_T(x)$, $D^u_T = \{(x_i,y_i),i=1,2,...,n_T^u\}$, which are available in advance as well as $n_T^l$ online labeled instances $D^l_T = \{(x_j,y_j), j=1,2,...,n_T^l\}$, which are sequentially arriving in the target domain, where $x_i,x_j \in \mathcal X_T$, $y_i,y_j \in \mathcal Y_T$.
  
  In this paper, we focus on homogeneous online transfer learning, where $\mathcal X_{S_i} = \mathcal X_{T} = R^{m}$, $\mathcal Y_{S_i} = \mathcal Y_T = \{1,2,...,K\}$ for $\forall i=1,2,...,n$. But the marginal distribution and the conditional distribution between domains are different, i.e., $P_{s_i}(x) \neq P_t(x), Q_{s_i}(y|x) \neq Q_t(y|x)$ for $\forall i=1,2,...,n$. We seek to find a new feature representation for each source domain and the target domain by utilizing transformation matrix $A_i \in R^{d \times m}$ for the  $i$-th domain, where $d$ is the dimension of the new feature representation.
  
  For an online transfer learning (OTL) task, our goal is to online learn a prediction function in the target domain from a sequence of labeled examples $D_T^l$. Specifically, during the online stage, at the $t$-th trial of online learning task, the learner receives an instance $x_t$. Firstly we use the current model to predict the label of $x_t$, secondly we will receive its true label,  thirdly we sequentially update the current model in the target domain based on the loss information between the true label and the predicted label. The key challenge is how to find a good representation that can reduce the distribution discrepancy between domains as well as how  to effectively transfer knowledge from source domains to the target domain.

  \subsection{Overall Framework}
  In this paper, we propose a method called \textit{Homogeneous Online Transfer Learning with Online Distribution Discrepancy Minimization} (HomOTL-ODDM), which is composed of an offline stage and an online stage. The structure of the proposed algorithm is shown in Fig. \ref{overall}.
  
   At the offline stage, for each source domain, we use JDA algorithm which utilizes labeled source data and unlabeled target data to get a satisfactory initial transformation matrix, denoted by $A_i$ for the $i$-th source domain. Then, for the $i$-th source domain, the source data are projected to a new feature space by $A_i$.  After that, an existing machine learning method is used to learn a classifier in  new feature space.
   Generally, there are $n$ transformation matrices and $n$ source classifiers in total.
  
   At the online stage, for $t$-trial target instance $x_t$, all the  transformation matrices are used to project $x_t$ to the corresponding feature space, denoted by $x^{p_1}_{t}, x^{p_2}_{t},...,x^{p_n}_{t}$. For the target data in all new feature space, an online learning method, e.g., multi-class PA algorithm \cite{b9} is used to get the  corresponding target classifiers, denoted by $f_{T_1},f_{T_2},...,f_{T_n}$. Totally, there are \textit{n} classifiers in the target domain. We combine source classifiers and target classifiers by allocating a weight to each classifier to build the final classifier $F$.
  
   At each round, we will update the target classifiers $f_{T_1},f_{T_2},...,f_{T_n}$ according to the prediction loss and the weights of all the classifiers by  the Hedge strategy\cite{b30}. Besides, to better reduce the discrepancy between domains as well as improve classification performance, we propose a method to online update transformation matrices. To speed up the algorithm, different from multi-class PA algorithm,  we set a \textit{time window} $T_w$ to control the update frequency of transformation matrices rather than  updating transformation matrices in each round.
  
  \subsection{Multi-class PA algorithm}
  
  In this paper, the source classifiers and target classifiers are all trained with Multi-class PA algorithm\cite{b9}, we will introduce this algorithm detailly. At the $t$-th round, the instance $x_t$ is associated a different weight vector with each of the $K$ labels. That is, the multiclass predictor is now paramterized by $w_t^1,...,w_t^K$, where $w_t^r \in R^d$. The output of the predictor is defined as,
   \begin{equation*}
  f(x_t) =  {((w_t^1 \cdot x_t),...,(w_t^K \cdot x_t)).}
   \end{equation*}
  The margin loss is used in Multi-class PA algorithm. Generally, at the $t$-th round we find the pair of indices $r_t,s_t$ which correspond to the largest violation of the margin constraints,
  \begin{align}
  r_t  = y_t   ,  \quad s_t = \arg\max_{s \neq y_t}  w_t^s \cdot x_t
  \label{rs}
  \end{align}
  where $y_t$ is the true label. Then, the loss in mutli-class case amounts to,
  \begin{equation}
  \begin{split}
  & l(x_t) = l(w_t^1,...,w_t^K;x_t,y_t) = \\
  &\left\{
                                      \begin{array}{ll}
                                      0  &  w^{r_t}_t \cdot x_t - w^{s_t}_t \cdot x_t \ge 1 \\
                                       1 - w^{r_t}_t \cdot x_t  + w^{s_t}_t \cdot x_t &  otherwise
                                      \end{array} \right.
  \end{split}
  \end{equation}
  
  Based on these definitions, the new weight vector ${w_{t+1}^1},...,w_{t+1}^K$ are the solution to the following constrained optimization problem,
  
  \begin{equation}
  \begin{split}
  w_{t+1}^1,...,w_{t+1}^K & = \arg\min_{w^1,...,w^K} \frac{1}{2}\sum_{i=1}^K||w^i-w_t^i||_2^2 \quad \\
  &s.t. \quad l(w^1,...,w^K;(x_t,y_t)) = 0
  \end{split}
  \end{equation}
  
  then we get the resulting update method for multi-class probelm,
  \begin{align}
  w_{t+1}^{r_t} = w^{r_t}_{t} + \tau_tx_t, w_{t+1}^{s_t} = w^{s_t}_{t} - \tau_tx_t  \\
  \tau_t = min \{C, \frac{l(x_t)}{||w_t^{r_t} \cdot x_t- w^{s_t}_t \cdot x_t||_2}\}
  \end{align}
  where C is the \textit{aggressiveness parameter} controlling the influence of the slack term on the objective function.
  \subsection{Offline Stage}
  
   Different from previous online transfer learning methods, we train classifiers in the new feature space instead of the original feature space. JDA algorithm\cite{b4} is a MMD-based method, which jointly adapts both the marginal distribution and conditional distribution in a principled dimensionality reduction procedure, and constructs new feature representation that is effective and
  robust for substantial distribution discrepancy. In our method, JDA algorithm is used to get the initial transformation matrix $A_i$ for the $i$-th source domain with the source data and unlabeled target  data. Then, the original source data is projected to a new feature space,  i.e., $\mathbf{X_{s_i}^{p}= {A_i}X_{s_i}}$. The same transformation process will also be used in the target domain at the online stage.
  
  To perform knowledge transfer, we train a classifier in each source domain using the multi-class PA algorithm(section 3.2) with the data in the new feature space $\mathbf{X^p_{s_i}}$. And then, we get the classifier $f_{S_i}$ for the $i$-th source domain.

   \subsection{Online Stage}
  
  As is shown in Fig. \ref{overall}, at the $t$-th round, after receiving the target instance $x_t$,  all the transformation matrices are used to project $x_t$ to the corresponding feature space, i.e., $\mathbf{x_{t}^{p_i} = {A_i}x_t}$ for the $i$-th transformation matrix. For the target data in each new feature space, an online learning method, e.g., multi-class PA algorithm is used to train the corresponding target classifier, denoted as $f_{T_i}$ for the $i$-th feature space. At the $t$-th round, it is  paramterized by $\boldsymbol {w_{i,t}^1,...,w_{i,t}^K}$. To combine source classifiers and target classifiers, we assign a weight to each classifier. The weight of the $i$-th source classifier is denoted by $u_i$, and the weight of the $i$-th target classifier is denoted by $v_i$. The initialize value of each weight is $\frac{1}{2n}$.

  Then we get the final classifier $F_t$ by Equation \ref{predict_1} and make prediction for $x_t$ by  Equation \ref{predict_2}, where the output of $F_t$ is a k-dimensional vector $\boldsymbol {F_t(x_t^{p_i})} = {(F_t^1,F_t^2,...,F_t^K)}$ and $\widehat{y_t} \in \{1,2...,K\}$ is the predicted label.
  \begin{equation}
  F_t =  \sum_{i=1}^nu_if_{S_i}^t(x^{p_i}_{t}) + v_if_{T_i}^t(x^{p_i}_{t})
  \label{predict_1}
  \end{equation}
  \begin{equation}
  \widehat{y_t} = \arg\max_{k} F_t^k
  \label{predict_2}
  \end{equation}
  After prediction, we will receive the true lable $y_t$, then a loss is computed between $y_t$ and $\widehat{y_t}$ and each target classifier is updated according to the multi-class PA algorithm. Besides, we except the weight of each classifier can be adjusted  dynamically. Inspired by the Hedge strategy \cite{b30}, we suggest the following updating scheme for adjusting the weights:
  \begin{equation}
  u_i =  \frac{u_i\beta^{z^i_t}}{\sum_{i=1}^nu_i\beta^{z^i_t}+\sum_{i=1}^nv_i\beta^{r^i_t}}
  \label{update_1_1}
  \end{equation}
  \begin{equation}
  v_i =  \frac{v_i\beta^{r^i_t}}{\sum_{i=1}^nu_i\beta^{z^i_t}+\sum_{i=1}^nv_i\beta^{r^i_t}}
  \label{update_1_2}
  \end{equation}
  where $\beta \in (0,1)$ is the weight discount parameter, $z^i_t = I[\widehat{y}_{s_i} \neq y_t]$, $\widehat{y}_{s_i}$ is the prediction label of the $i$-th source classifier, similarly, $r^i_t = I[\widehat{y}_{t_i} \neq y_t]$, $\widehat{y}_{t_i}$ is the prediction label of the $i$-th target classifier.

  \begin{algorithm}[tbp]
  \caption{\textbf{HomOTL-ODDM} }
  \label{a1}
  \begin{algorithmic}[1]
  \REQUIRE the classifiers of source domains $f_{S_1},f_{S_2},...,f_{S_n}$, the initial transformation matrices $A_1,A_2,...,A_n$, initial trade-off parameter $C$, $\mu$, and the weight discount $\beta \in (0,1)$.
   $v_i=\frac{1}{2n}$, $u_i=\frac{1}{2n}$, $i= 1,2,...n$. \\
          \quad \quad \quad  $w_{i,0}^j =\mathbf{0}$, $i= 1,2,...n, j = 1,2,...,K$.
  \FOR{$t = 1,2,...,T$}
     \STATE Receive target domain instance ${{x}_{t}}\in {\mathcal X_T}$.
     \STATE Project $x_t$ to corresponding feature space: $x^{p_i}_{t} = A_i^tx_t$
     \STATE Predict $\widehat{y_t}$  by Equations \ref{predict_1} and \ref{predict_2}.
     \STATE Receive true lable: $y_t \in \mathcal Y_t$.
     \FOR{$i=1,2,...,n$}
        \STATE Update the weights $v_i, u_i$ by Equations \ref{update_1_1} and \ref{update_1_2}.
        \STATE Calculate $s_t$ , $r_t$ by equation \ref{rs}
        \STATE  Suffer loss: ${{l}_i}(x^{p_i}_t)={[0,1 - w^{r_t}_{i,t} \cdot x^{p_i}_t  + w^{s_t}_{i,t}  \cdot x^{p_i}_t}]_+$
        \IF{${{l}_i}(x^{p_i}_t)>0$ }
           \STATE  $w_{i,t+1}^{r_t} = w^{r_t}_{i,t} + \tau_t  x^{p_i}_t, w_{i,t+1}^{s_t} = w^{s_t}_{i,t} - \tau_t  x^{p_i}_t$
           \STATE $\tau_t = min \{C, \frac{{{l}_i}(x^{p_i}_t)}{||w_{i,t}^{r_t} \cdot x^{p_i}_t - w^{s_t}_{i,t} \cdot x^{p_i}_t||}\}$
        \ENDIF
     \ENDFOR
     \IF{ t \textit{mod} $T_w$ = 0}
     \STATE Update each transformation matrix $A_i$ by Equation \ref{A_update_final}
     \ENDIF
  \ENDFOR
  \ENSURE $F_t = \sum_{i=1}^nu_if_{S_i}^t(x^{p_i}_{t}) + v_if_{T_i}^t(x^{p_i}_{t})$.
  \end{algorithmic}
  \end{algorithm}
  
  \subsection{The Update of Transformation Matrix}

  At the online stage, besides updating target classifiers and the weights of each classifier, we also update the transformation matrices. On one hand, the transformation matrices are got by target unlabeled data, which don't involve any label information, so utilizing online labeled target data is helpful for improving classification performance. On the other hand, as target labeled data increase, we can use more and more target data so that the measure of distribution discrepancy between domains can become more accurate. In this case, the transformation matrices will be closer to the optimal solution and the distribution discrepancy in new feature space between domains can be smaller.

  \begin{table*}[tbp]
   \caption{Experimental dataset details}
    \centering
      \begin{tabular}{p{1.7cm}<{\centering}p{1.6cm}<{\centering}p{1.6cm}<{\centering}p{1.6cm}<{\centering}p{1.6cm}<{\centering}p{4.0cm}<{\centering}p{2.8cm}<{\centering}}
  
      \toprule
      Tasks & SD size  & TD size & Dimensions & Class &Source domains(SD) & Target domains(TD)  \\
      \midrule
      PIE1 & 8222  & 3332 & 1024  & 68 & PIE2,PIE3,PIE4,PIE5   & PIE1\\
      PIE2 & 9925  & 1629 & 1024  & 68 & PIE1,PIE3,PIE4,PIE5   & PIE2\\
      PIE3 & 9922  & 1632 & 1024  & 68 & PIE1,PIE2,PIE4,PIE5   & PIE3\\
      PIE4 & 8225  & 3329 & 1024  & 68 & PIE1,PIE2,PIE3,PIE5   & PIE4\\
      PIE5 & 9922  & 1632 & 1024  & 68 & PIE1,PIE2,PIE3,PIE4   & PIE5\\
      Amazon & 1575  & 958  & 800  & 10 & Caltech-256,DSLR,Webcam   & Amazon\\
      Caltech-256 & 1410  & 1123 & 800  & 10 & Amazon,DSLR,Webcam   & Caltech-256\\
      DSLR & 2376  & 157  & 800  & 10 & Amazon,Caltech-256,Webcam   & DSLR\\
      Webcam & 2238  & 295  & 800  & 10 & Amazon,Caltech-256,DSLR   & Webcam\\
      \bottomrule
      \end{tabular}%
  \label{t_dataset}
  \end{table*}%

  We update each transformation matrix $A_i$ of the $i$-th domain by reducing both the marginal distribution discrepancy and conditional distribution discrepancy simultaneously. Empirical MMD  is adopted as the distance measure to compare different distributions.The marginal distribution discrepancy is computed as:
  \begin{equation}
  ||{A_i\bar{c}_{S_i} - A_i\bar{c}_T}||^2_2
  \label{margin}
  \end{equation}
  where  $\bar{c}_{S_i} = \frac{1}{n_{S_i}} \sum_{j=1}^{n_{S_i}} x_j$ is the mean of the $i$-th source domain data, and  $\bar{c}_T = \frac{1}{t}\sum_{j=1}^t{x_j}$ is the mean of the target data till  round \textit{t}. By minimizing Equation \ref{margin}, the marginal distributions between domains are drawn close under the new feature representation.
  
  As $Q(x|y)$ and $Q(y|x$) can be quite involved according to the sufficient statistics when sample sizes are large \cite{b4}, we use the class conditional distribution $Q(x|y)$ to approximate $Q(y|x)$. The distance between conditional distribution $Q_{s}(x|y)$ and $Q_t(x|y)$ is measured as:
  \begin{equation}
  \sum_{k=1}^K||{A_i\bar{c}_{S_i}^{k} - A_i\bar{c}_T^{k}}||^2_2
  \label{condition}
  \end{equation}
  where $\bar{c}_{S_i}^{k} = \frac{1}{n^{(c)}}{\sum_{x_j \in D_{S_i}^{(c)}}} x_j$ represent the mean of examples belonging to class $k$ in the $i$-th source domain, $D_{S_i}^{(c)} = \{x_j:x_j \in D_{S_i} \wedge y(x_j)=k\}$, $n^{(c)} = |D_{S_i}^{(c)}|$. Correspondingly, $\bar{c}_T^{k} = \frac{1}{m^{(c)}}{\sum_{x_j \in D_{T}^{(c)}}} x_j$ represent the mean of the data whose label are $k$ in the target domain, $D_{T}^{(c)} = \{x_j:x_j \in D_{T} \wedge y(x_j)=k\}$, $m^{(c)} = |D_{T}^{(c)}|$.

  We aim to simultaneously minimize the discrepancy of both marginal and conditional distributions while reserving previous information. Thus, we incorporate Equations \ref{margin} and \ref{condition}, which leads to the following optimization problem:
  \begin{equation}
  {A^{t+1}_{i}} = \arg\min_A  ||{A-A^t_i}||_F^2 + \mu \sum_{k=0}^K||{A(\bar{c}_{S_i}^{k} - \bar{c}_T^{k})}||^2_2
  \label{optimiza}
  \end{equation}
  where $||{A(\bar{c}_{S_i} - \bar{c}_T)}||^2_2$  is denoted by $ ||{A(\bar{c}_{S_i}^{0} - \bar{c}_T^{0})}||_2^2$, $\mu$ is the regularization parameter. Solving the above optimization problem, we will get the updating method for $A_i$:
  \begin{equation}
    {A^{t+1}_i} = {A_i^t}(I + \mu \sum_{k=0}^{K}{(\bar{c}_{S_i}^{k} -{\bar{c}_T^{k}})(\bar{c}_{S_i}^{k} -{\bar{c}_T^{k}})^T)^{-1}}
    \label{A_update_final}
  \end{equation}
  when $I + \mu \sum_{k=0}^{K}(\bar{c}_{S_i}^{k} -{\bar{c}_T^{k}})(\bar{c}_{S_i}^{k} -{\bar{c}_T^{k}})^T$ is reversible, we will update $A_i$ using Equation \ref{A_update_final}, otherwise, we will use  its pseudo-inverse  instead.

  As we can see, matrix inversion for updating $A_i$  would cost a lot of time.  To decrease time cost and enhance the robustness of the proposed method,  we set a \textit{time window} $T_w$ for all transformation matrices such as $T_w$= 50 to control the updating frequency. We will update $A_i$ once every  \textit{time window}.

  Combing the above   update rules of classifiers , weights and transformation matrices, we propose homogeneous online transfer learning with online  domain adaptation algorithm (HomOTL-ODDM). The pseudocode of the proposed method is shown in algorithm \ref{a1}.

  \subsection{ Theoretical Analysis}
  
  \begin{table*}[tbp]
   \caption{Mistakes rates(\%) on 9 cross-domain image tasks}
    \centering
      \begin{tabular}{cp{1.7cm}<{\centering}cp{1.7cm}<{\centering}cp{1.7cm}<{\centering}|cp{2cm}<{\centering}}
  
      \toprule
      Task & PA  & PAIO & HomOTL-I & HomOTL-II & HomOTLMS &  \tabincell{c}{HomOTL-ODDM\\(fixed)} &  \tabincell{c}{HomOTL-ODDM} \\
      \midrule
      PIE1 & $78.45 \pm 0.65$& $40.10 \pm 0.80$ & $63.08\pm0.37$ & $64.46\pm0.34$ & $67.10\pm0.71$ &  $13.92 \pm 0.65$   & $\mathbf{10.39 \pm 1.09}$\\
      PIE2 & $90.99\pm0.62$  & $53.82\pm 0.88$  & $62.15\pm0.31$ & $62.13\pm0.31$ & $69.33\pm0.75$ &  $30.87 \pm 0.55$   & $\mathbf{26.82 \pm 1.86}$\\
      PIE3 & $90.66\pm0.47$  & $52.49\pm0.96$   & $58.93\pm0.36$ & $58.91\pm0.33$ & $69.97\pm1.22$ &  $25.57 \pm 1.47$   & $\mathbf{20.35 \pm 2.87}$\\
      PIE4 & $75.06\pm0.74$  & $31.49\pm0.57$   & $45.22\pm0.22$ & $45.25\pm0.21$ & $60.97\pm0.86$ &  $19.46 \pm 0.54$   & $\mathbf{8.91 \pm 0.95}$\\
      PIE5 & $88.44\pm0.66$  & $57.53\pm1.31$   & $72.04\pm0.55$ & $72.74\pm0.57$ & $74.89\pm1.33$ &  $33.75 \pm 1.21$   & $\mathbf{22.85 \pm 1.78}$\\
      Amazon & $45.66\pm1.26$  & $\mathbf{39.13\pm1.29}$ & $42.89\pm1.78$ & $41.37\pm1.04$ & $39.57\pm0.96$ & $44.23 \pm 1.12$	& $42.05 \pm 1.63$\\
      Caltech-256    & $58.54\pm1.23$  & $51.02\pm1.20$   & $56.88\pm1.18$ & $52.64\pm1.18$ & $50.94\pm1.21$ &  $51.22 \pm 1.36$   & $\mathbf{50.65 \pm 1.15}$\\
      DSLR    & $52.64\pm3.24$  & $29.68\pm1.98$   & $30.27\pm2.78$ & $31.00\pm1.38$ & $24.55\pm2.19$ &  $22.86 \pm 2.09$   & $\mathbf{23.09 \pm 3.27}$\\
      Webcam    & $43.54\pm1.71$  & $34.83\pm1.65$   & $34.90\pm2.43$ & $36.36\pm1.92$ & $30.22\pm1.85$ &  $27.86 \pm 2.36$   & $\mathbf{26.17 \pm 2.89}$\\
      \bottomrule
      Average    & $69.33\pm1.17$  & $43.34\pm1.18$ & $51.81\pm1.11$ & $51.65\pm0.80$ & $54.17\pm1.23$ & $29.97\pm1.26$ & $\mathbf{25.69 \pm 1.94}$\\
      \bottomrule
      \end{tabular}%
  \label{t_result}
  \end{table*}%
  
  Similar to \cite{b7}, we provide the mistake bound of the algorithm HomOTL-ODDM as follows.
  
  \newtheorem{thm}{\bf Theorem}
  \begin{thm} \label{thm1}
  Let us denote $M$ as the number of mistakes made by the algorithm. By choosing $\beta = \sqrt{M_{min}}/(\sqrt{M_{min}} + \sqrt{\ln2n})$, we have M bounded by
  \begin{equation}
  M \leq M_{min} + \frac{3}{2}\sqrt{\ln{2n}M_{min}} + \ln{2n}
  \label{M_res}
  \end{equation}
   \begin{equation*}where \quad
  M_{min} = min\{M_s^i,M_t^i\}, \quad
  M_s^i = \sum_{t=1}^T z^i_t, \quad
  M_t^i = \sum_{t=1}^T r^i_t \quad
  \end{equation*}
  \end{thm}

  \begin{proof}
  From the Theorem 2 in Hedge Algorithm \cite{b30}, if we set all the initial weights equally to be $1/2n$ for each source and target classifier, we can obtain the following mistake bound for $M$,
  \begin{equation}
  M \leq \frac{M_{min}\ln(1/\beta)+\ln2n}{1-\beta}
  \label{M_bound}
  \end{equation}
  It can be shown that $\ln(1/\beta) \leq (1-\beta^2)/2\beta$ for $\beta \in (0,1]$. If we apply this equation to \ref{M_bound}, then we can obtain
  \begin{equation}
  \frac{M_{min}\ln(1/\beta)+\ln2n}{1-\beta} \leq \frac{M_{min}(1+\beta)}{2\beta} + \frac{\ln2n}{1-\beta}
  \label{M_temp}
  \end{equation}
  By substituting $\beta = \sqrt{M_{min}}/(\sqrt{M_{min}} + \sqrt{\ln2n})$ into \ref{M_temp}, we obtain \ref{M_res}.
  \end{proof}
  
  Theorem 1 provides an upper bound of the mistake for our algorithm. The result in theorem 1 implies that the learner's average mistake per round  can never be much larger than that of the best pure strategy. Specially, in HomOTL-ODDM, the classifiers are trained in new feature space, where the distribution discrepancy between domains is reduced, so it is expected to get fewer mistakes in each domain than in original feature space.

  \section{Experiment}

  In this paper, nine image classification tasks are established on two widely adopted benchmark datasets \textbf{PIE} and \textbf{Office+Caltech}(refer to Table \ref{t_dataset})  to evaluate the proposed method. Codes will be available at  \url{https://github.com/yaoyueduzhen/HomOTL-ODDM}.
  
  \subsection{Datasets}

  \textbf{PIE} has 68 individuals with 41638 faces of size 32*32. In this paper, \textbf{PIE1}(C05, left pose), \textbf{PIE2}(C07, upward pose), \textbf{PIE3}(C09, downward pose), \textbf{PIE4}(C27, frontal pose) and  \textbf{PIE5}(C29, right pose) are chosen. \textbf{Office} dataset is composed by three real-world object domains: Amazon, Webcam and DSLR. It has 4652 images with 31 object categories. \textbf{Caltech-256} is a standard database for object recognition, which has 30607 images and 256 categories.
  
  In the paper, we adopt public \textbf{Office+Caltech} datasets\cite{b28}. SURF features are extracted and quantized into an 800-bin histogram with codebooks computed with K-means on a subset of images from Amazon. Then the histograms are standardized by z-score. Specially, we will have four domains: \textbf{Caltech-256, Amazon, DSLR} and \textbf{Webcam}. We choose one domain as the target domain, the remain domains are used as source domains. In this way, we can generate four learning tasks in dataset PIE and five learning tasks in dataset Office+Caltech.
  
  \subsection{Baseline Methods}

  To evaluate the performance of the proposed algorithm, we compare our algorithm with several state-of-the-art methods,  and for fair comparison, all of these methods adopt the mutil-class PA algorithm as classifiers.
  
  \begin{itemize}
  \item \textbf{PA \cite{b9}:}  A classical online learning method without exploiting any knowledge from source domain.
  \item \textbf{PAIO \cite{b7}:}  A variant of PA algorithm  by initializing PA with a classifier trained with the whole source domain.
  \item \textbf{HomOTL-I \cite{b5}:} A single-source online transfer learning algorithm, which adopts a loss-based weight update strategy.
  \item \textbf{HomTOL-II \cite{b5}:} A single-source online transfer learning algorithm, which adopts a mistake-driven weight update strategy.
  \item \textbf{HomOTLMS \cite{b7}:} A recently proposed mutilple-source online transfer learning method that combines source classifiers and target classifier adaptively.
  \item \textbf{HomOTL-ODDM(fixed):} A variant of our proposed method, i.e., we just use the initial transformation matrices without updating at the online stage
  \end{itemize}
  
  \subsection{Implementation Details}
  
  We adopt the mutil-class PA algorithm to run on the source dataset and adopt the average classifier as the source classifier in each domain, which generally enjoys better generalization ability \cite{b13}. Futher, we draw 20 times of random permutation of the instances in the target domain in order to obtain stable results by averaging over the 20 trials. Considering that HomOTL-I and HomOTL-II are single-source online transfer learning methods, we combine all the instances in different source domains as a single source domain for HomOTL-I and HomOTL-II, with hyper-parameter $C$ = 5 and $\beta = \sqrt{T}/(\sqrt{T} + \sqrt{ln2})$.
  
  In our method, we set hyper-parameter in JDA algorithm as $m' = 100, T = 10$ and 1) $\lambda = 0.1$ for the \textbf{PIE} datasets, 2) $\lambda = 1$ for the \textbf{Office+Caltech} datasets. Besides, we also set $C$ = 5, $\beta = \sqrt{T}/(\sqrt{T} + \sqrt{ln2})$, $\mu = 1$ and 1) $T_w$ = 50 for \textbf{PIE} datasets, 2) $T_w$ = 10 for \textbf{Office+Caltech} datasets. For performance evaluation, from each of these nine tasks, we randomly sample  30\% from the target domain to form the unlabeled subset $X_T^u$, and the remaining 70\% in the target domain to form the online labeled subset $X_T^l$.
  
  In this paper, we evaluate the predictive accuracy of online learning methods by measuring the standard mistake rate on the labeled subset, which is defined as,
  \begin{equation}
  mistake \quad rate = \frac{|x:x \in {D_T^l} \bigwedge \widehat{y_t} \neq y_t|}{|x:x \in D_T^l|}
  \end{equation}

  \subsection{Experiment Results}
   The classification mistake rates of HomOTL-ODDM and the five baseline methods on the nine cross-domain image tasks are illustrated in Table \ref{t_result}. Several observation can be drawn from the experimental results. First of all, PA algorithm performs the worst in most of the tasks, which indicates that transferring knowledge from source domains can improve the target performance. Second, the average classification mistake rate of HomOTL-ODDM is \textbf{25.69\%} compared to the best baseline method PAIO, i.e., a significant error reduction of \textbf{17.65\%}.  This verifies that HomOTL-ODDM can construct more effective representation for cross-domain classification tasks. Third, HomOTL-ODDM outperforms HomOTLMS. A major limitation of HomOTLMS is that the distribution discrepancy between domains is not explicitly reduced. HomOTL-ODDM avoids this limitation and achieves much better results. Last but not least, HomOTL-ODDM achieves much better performance than HomOTL-I and HomOTL-II, which are two single-source online transfer learning methods, because HomOTL-ODDM not only reduce distribution discrepancy between domains but also transfer knowledge from multiple source domains.

   Fig. \ref{fig_mistakes} also shows the details of average mistake rates varying over the learning processes. We only report results on task PIE1 and Amazon, while similar trends on all other datasets are not shown due to space limitation.  The observations show that the proposed algorithm achieve the best performance compared with baseline methods.

  \begin{figure}[htbp]
  \centering
  \subfigure[PIE1]{
  \includegraphics[width=0.46\linewidth]{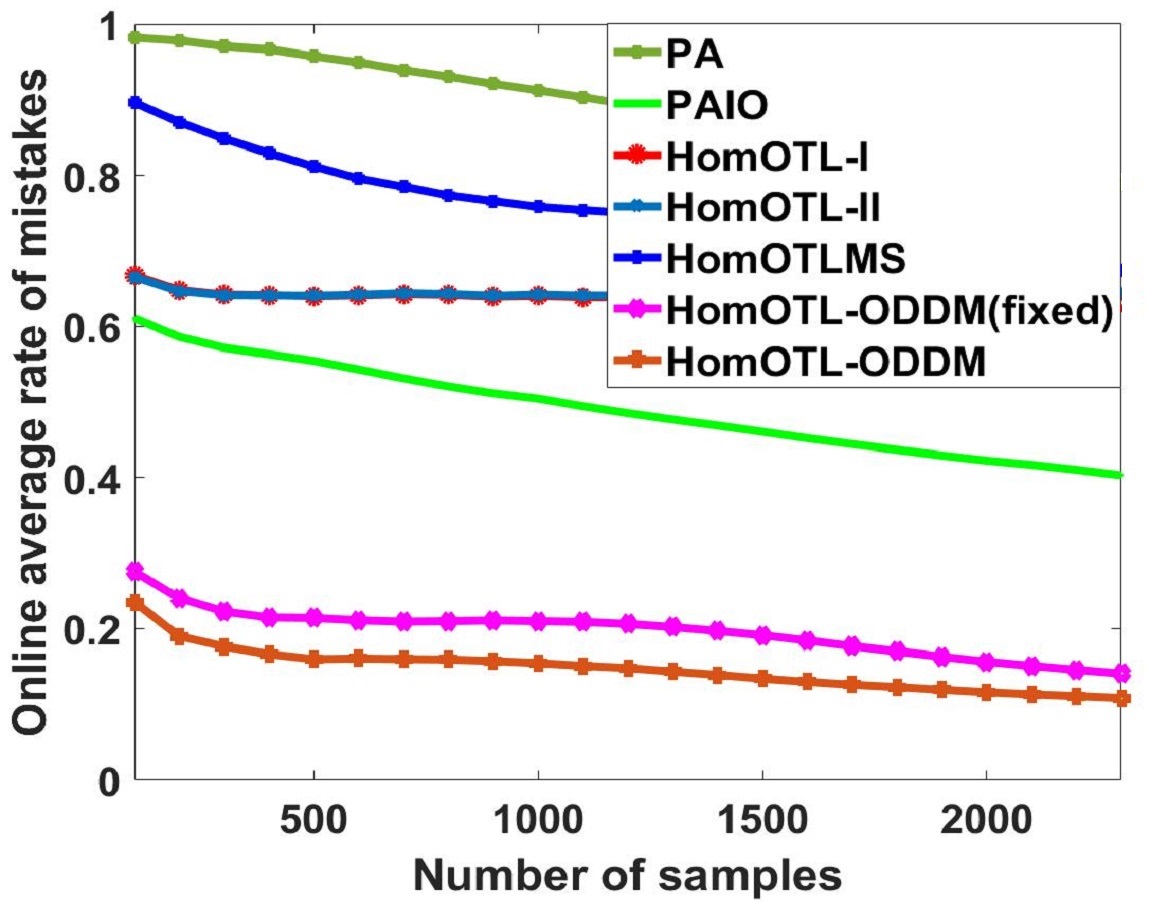}
  }
  \subfigure[Amazon]{
  \label{fig:subfig:a}
  \includegraphics[width=0.46\linewidth]{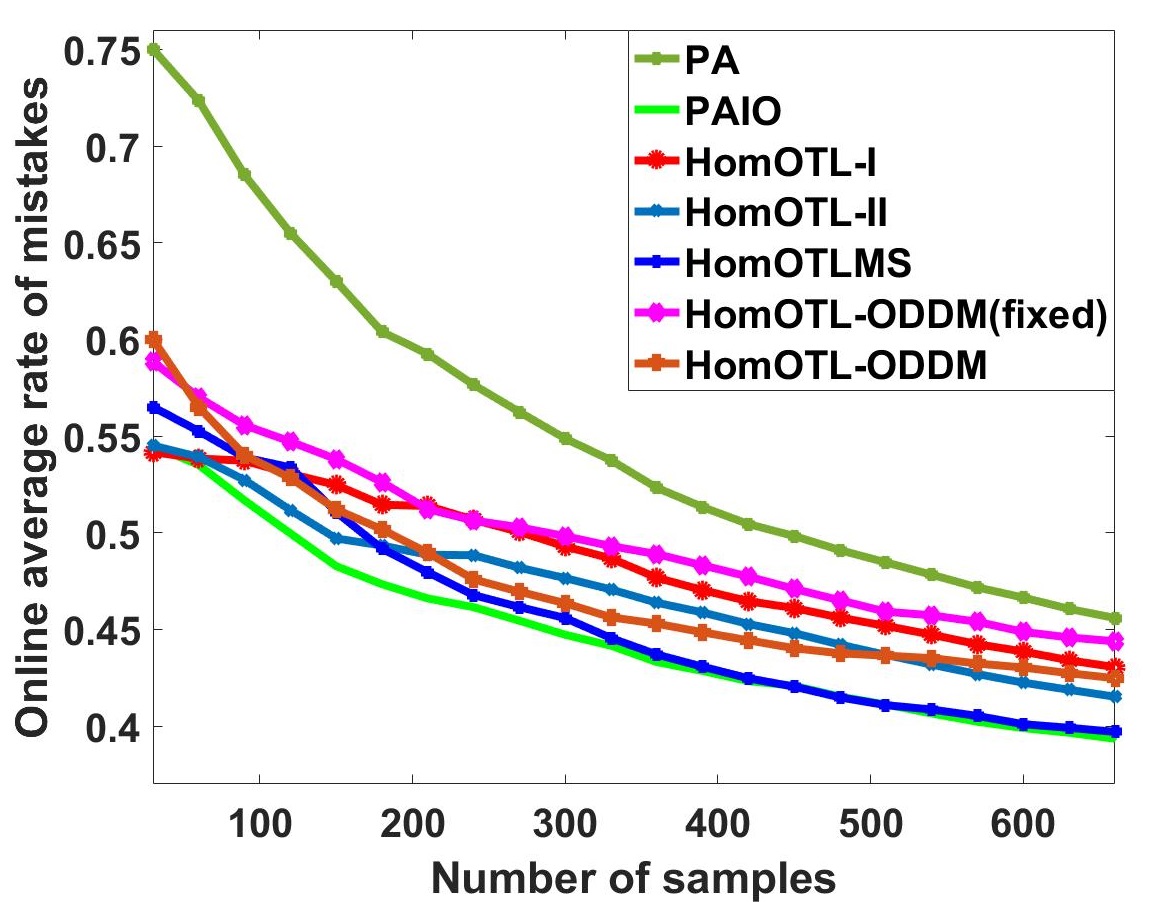}
  }
  \caption{Evaluation of online mistake rates varying over learning process}
  \label{fig_mistakes}
  \end{figure}

  \subsection{Effectiveness Verification}
  
  \textbf{Effectiveness of Updating transformation Matrix:} 
  Table \ref{t_result} shows the result of HomOTL-ODDM(fixed) compared with HomOTL-ODDM. As we can see, 1) both methods perform better than all the baselines, which indicates that both methods can improve transfer performance by reducing distribution discrepancy between domains. 2) HomOTL-ODDM gets better results compared with HomOTL-ODDM(fixed), and the average mistakes rates is reduced by 4.28\%, which reveals that updating transformation matrices using online labeled target data is effective to improving transfer performance.

  \textbf{Distribution Discrepancy:} We evaluate the distribution discrepancy varying over the learning processes in task PIE1(refer to Fig. \ref{fig_MMD_1}) and Amazon(refer to Fig. \ref{fig_MMD_2}). We compute the aggregate MMD distance of HomOTL-ODDM on their induced embeddings by Equation \ref{optimiza}. Note that, in order to compute the true distance in both the marginal and conditional distributions between domains, we have to use the groundtruth labels for target  data. However, the groundtruth labels are only used for verification, not for learning procedure.

   Above all, we have three observations.
   1) As target data increases, updating transformation matrices can help extract a more effective and robust representation, which can reduce the distribution discrepancy between domains as well as improve the transfer performance. 2) When there is small discrepancy between domains, e.g., task Amazon,  HomOTL-ODDM can achieve the optimal distance after receiving a small number of examples. 3) Considering previous classification results, we find that when there is larger discrepancy between domains, HomOTL-ODDM can achieve better transfer performance improvement.

  \begin{figure}
    \centering
    \subfigure[PIE]{
    \label{fig_MMD_1}
    \includegraphics[width=0.44\linewidth]{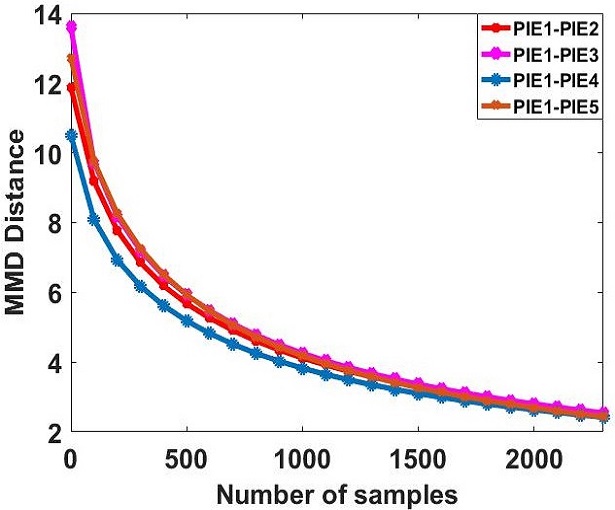}
    }
    \subfigure[Amazon]{
    \label{fig_MMD_2}
    \includegraphics[width=0.44\linewidth]{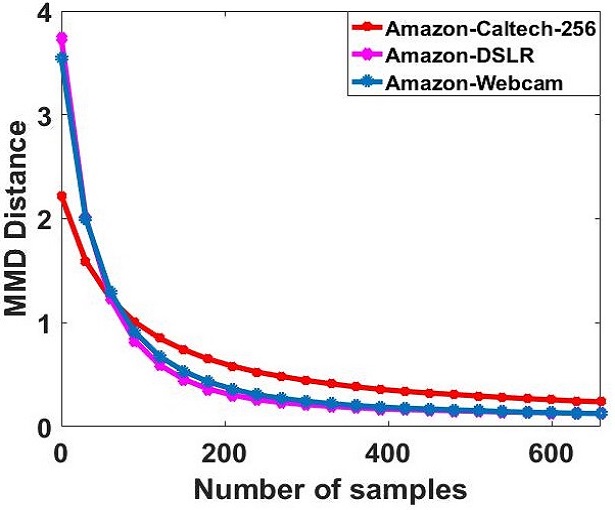}
    }
    \caption{Distribution discrepancy varying over the learning processes}
    \label{fig_MMD}
  \end{figure}

  \subsection{Parameter Sensitivity}
  
  We conduct sensitivity analysis to validate that HomOTL-ODDM can achieve optimal performance under a wide range of parameter values.
  
  We run HomOTL-ODDM with varying values of $C$. We only report results on task PIE1 and Webcam, while similar trends on all other datasets are not shown due to space limitation.
   Fig. \ref{fig_para_c} evaluates the online prediction performance of the compared algorithms with varied $C$ values across all the tasks. Several observations can be drawn from the results. First of all, it is clear that the proposed method is significantly more effective in most cases. Second, among all the compared algorithms, we observe that HomOTL-ODDM always achieve the best performance when $C$ is sufficiently large($C >2$), which indicates that a large $C$ can improve the transfer efficiency. In the experiments, we set $C$ to be 5 for all the algorithms.

  \begin{figure}
  \centering
  \subfigure[PIE1]{\label{fig:subfig:b}
  \includegraphics[width = 0.46\linewidth]{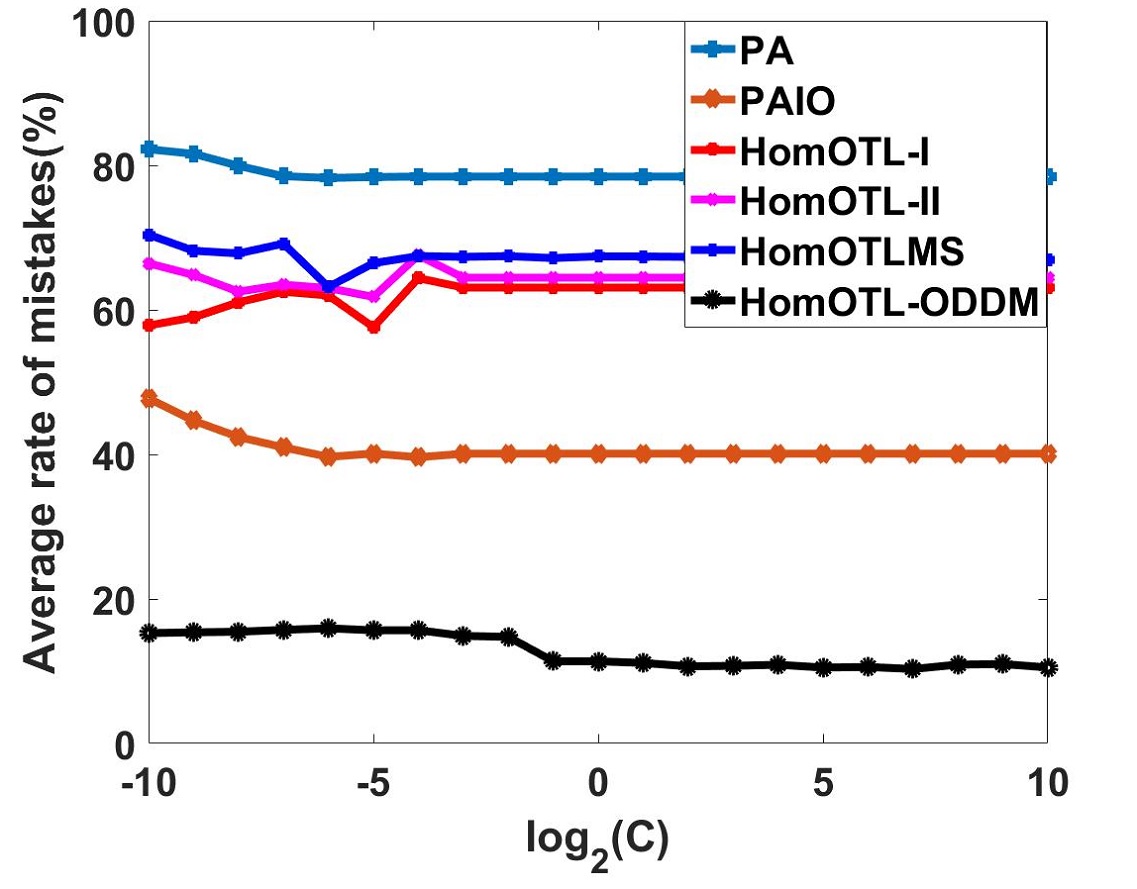}
  }
  \subfigure[Webcam]{
  \label{fig:subfig:a}
  \includegraphics[width = 0.46\linewidth]{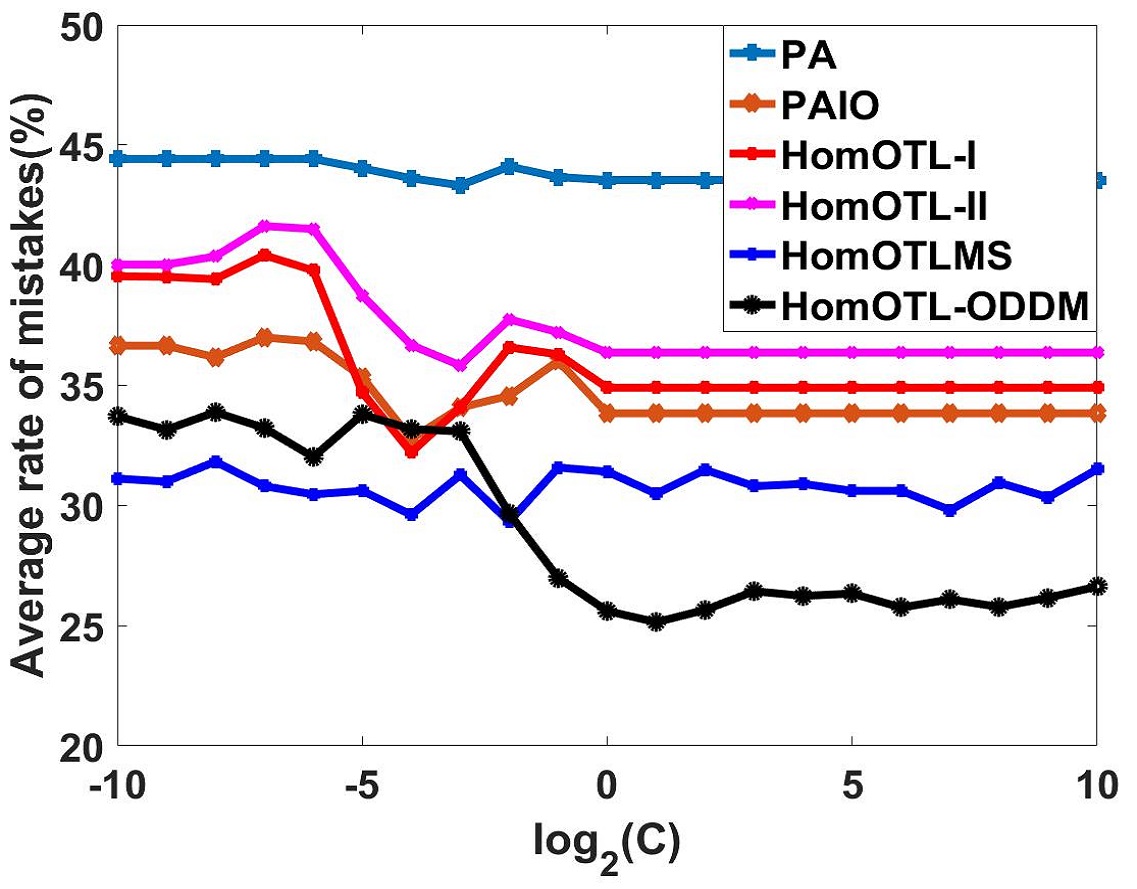}
  }
  \caption{Parameter sensitivity of $C$}
  \label{fig_para_c}
  \end{figure}

  We run HomOTL-ODDM with varying values of $T_w$. As we can see, the smaller $T_w$ is, the more frequently transformation matrices will be updated, which will cost a lot of time as well as perform a poor performance. When $T_w$ becomes large, the update frequency will decrease, it is possible that it is not effective to reduce the distribution discrepancy between domains. We plot average mistakes w.r.t. different values of $T_w$ in Fig. \ref{fig_para_tw}, which indicates that $T_w \in [30, 200]$ can be optimal parameter values.

  \begin{figure}[h]
  \centering
  \subfigure[PIE datasets]{
  \label{fig:subfig:a}
  \includegraphics[width=0.46\linewidth]{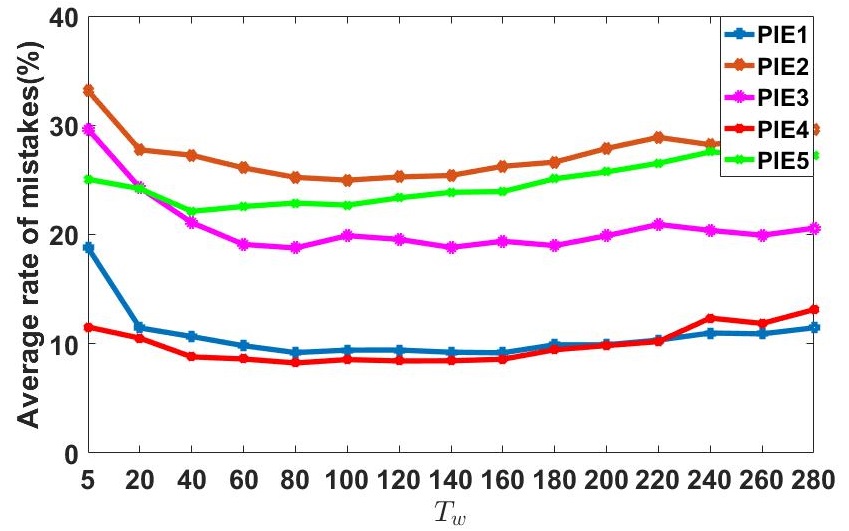}
  }
  \subfigure[Office+Caltech datasets]{\label{fig:subfig:b}
  \includegraphics[width=0.46\linewidth]{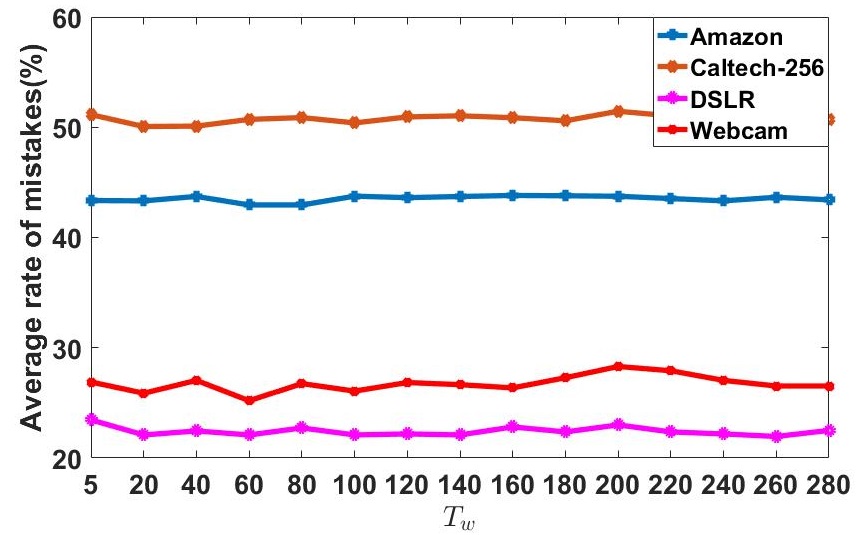}
  }
  \caption{Parameter sensitivity of $T_w$}
  \label{fig_para_tw}
  \end{figure}

  \section{Conclusion}
  
  In this paper, we propose a method called HomOTL-ODDM. We focus on online transfer learning  with multiple source domains and use the Hedge strategy to leverage knowledge from source domains. We seek to find a new feature representation by projecting the original source and target data onto a new space, where the distribution discrepancy between domains can be online reduced simultaneously.  Furthermore, We provide the mistake bound of the proposed method and conduct extensive experiments. The experimental results show that HomOTL-ODDM significantly outperforms several state-of-the-art methods in a variety of cross-domain problems.
  
\bibliography{ecai}
\end{document}